\documentclass[11pt]{article}

\usepackage[T1]{fontenc}
\usepackage[utf8]{inputenc}
\usepackage[margin=1in]{geometry}
\usepackage{amsmath}
\usepackage{amssymb}
\usepackage{amsthm}
\usepackage{array}
\usepackage{algorithm}
\usepackage{algpseudocode}
\usepackage{booktabs}
\usepackage{enumitem}
\usepackage{graphicx}
\usepackage{microtype}
\usepackage[colorlinks=true,linkcolor=blue,citecolor=blue,urlcolor=blue]{hyperref}

\hypersetup{
  pdftitle={Exact Regular-Constrained Variable-Order Markov Generation via Sparse Context-State Belief Propagation},
  pdfauthor={Francois Pachet}
}

\newtheorem{definition}{Definition}
\newtheorem{proposition}{Proposition}
\newtheorem{observation}{Observation}

\newcommand{\vocab}{\mathcal{V}}

\newcommand{\constraints}{\mathcal{C}}
\newcommand{\contextset}{\mathcal{T}}
\newcommand{\aug}{\mathcal{G}}

\newcommand{\ctx}{c}
\newcommand{\pot}{\psi}
\newcommand{\vo}{\mathrm{VO}}
\newcommand{\fo}{\mathrm{1}}
\newcommand{\Prvo}{P_{\vo}}
\newcommand{\Prfo}{P_{\fo}}
\newcommand{\allowed}{A}
\newcommand{\suffix}{\operatorname{suffix}}
\newcommand{\canon}{\operatorname{canon}}

\title{Exact Regular-Constrained Variable-Order Markov Generation via Sparse Context-State Belief Propagation}
\author{%
  Fran\c{c}ois Pachet\\
  LIP6, Sorbonne Universit\'e\\
  Ynosound\\
  \texttt{pachet@gmail.com}
}
\date{May 2026}

\begin{document}

\maketitle

\begin{abstract}
Variable-order Markov models generate sequences over a finite alphabet by
conditioning each symbol on the longest available suffix of the generated
history. Regular constraints, by contrast, describe finite-horizon control
requirements by an automaton: fixed positions, forced endings, metrical
patterns, and forbidden copied fragments are all special cases. Existing exact
methods already handle regular constraints with belief propagation for
first-order Markov chains. The contribution here is the variable-order
extension: identifying the state space on which the existing BP-regular
machinery must be run when the generator is a variable-order/backoff model. A
first-order constraint layer can enforce useful support conditions, but it
computes future mass after merging histories that a variable-order generator
deliberately keeps distinct. We formalize this mismatch and give the sparse
construction obtained by replacing the first-order Markov state with the
observed context state, then taking the standard product with the regular
constraint automaton. For a fixed trained context graph and automaton,
inference is linear in the sequence horizon; in general it is polynomial in the
number of reachable product edges. This gives the correct variable-order
distribution conditioned on regular constraints without expanding to all
\(K\)-tuples. The same finite-source interface supports reversible data
augmentation by inverse count lookup, matching materialized transposition
augmentation without storing transformed corpora. We also separate exact BP
inference from generation-time backoff policies, such as singleton avoidance,
whose stochastic semantics must be made explicit if exactness is claimed.
\end{abstract}

\section{Introduction}

Sequence generators often need to satisfy two goals that pull in different
directions. The first goal is local statistical coherence: each generated symbol
should be plausible given the context already produced. This is the strength of
variable-order Markov models. Rather than conditioning on a fixed previous
symbol, the generator searches for the longest suffix of the current history
that has observed continuations and samples from that distribution. This
mechanism can preserve local patterns that a first-order model would blur. In
interactive musical systems, however, variable order is more than a predictive
device. Backing off from a long context to a shorter one is also a way to avoid
being trapped by sparse or deterministic high-order continuations. This matters
because musical style is often specified by a few examples rather than by the
large homogeneous corpora on which transformer-style models thrive. In that
regime, variable-order Markov models remain attractive: they are data efficient,
interpretable, and close to the musical idea of reusing local stylistic material
without necessarily copying it verbatim. The musical intent is often to find a
continuation that is plausible under the learned style without merely replaying
the unique continuation of a memorized fragment. This style--innovation tradeoff
has been used explicitly to evaluate musical sequence models: fixed-order Markov
models tend to increase literal borrowing as the order grows, while
variable-order strategies can back off from over-specific contexts and reduce
copying while preserving local stylistic evidence
\cite{sakellariou2017maxent}.

The second goal is explicit control. One may require the generated sequence to
start with a specified symbol, end in a target set, contain a particular symbol
at a fixed position, avoid a forbidden substring, or stop exactly at a special
end symbol. These requirements can be expressed as regular languages over the
generated alphabet. Positional constraints are the simplest case; plagiarism or
MAXORDER constraints are another important case, because they forbid long
substrings copied from the source material
\cite{papadopoulos2014avoiding}. Such constraints are not local in the
left-to-right generation process: a choice made now may determine whether a
requirement several steps ahead remains reachable.

Small creative corpora create a related practical issue. One may want to keep
the training material small and explicit, while still exploiting reversible
transformations such as pitch transpositions, register shifts, or rhythmic
variants. Materializing every transformed copy multiplies the count table and
context graph before any constraint product is formed. The construction below
treats augmentation as a source-level row computation: counts are accumulated
by inverse lookup through the transformation family, and regular BP sees the
same sparse stochastic source interface.

The Markov constraints line of work addressed this problem for finite-length
Markov chains \cite{pachet2011markov,pachet2011finite}. Pachet and Roy's
formulation is expressive: Markovianity and additional user requirements are
combined inside a constraint satisfaction or optimization framework. Later
belief-propagation constructions give exact sampling for first-order Markov
chains with unary or regular constraints \cite{papadopoulos2015exact}, with
related extensions for meter and other musical controls
\cite{roy2013meter,pesant2004regular}. These results are the starting point of
the present paper. We do not reintroduce regular constraints, nor claim BP for
regular languages as new. The missing piece is the variable-order state
abstraction: the BP-regular construction is exact only for the stochastic source
whose state is used in the dynamic program.

The same point can be stated in weighted-automaton language. If a stochastic
source is represented as a weighted finite-state automaton, regular
conditioning is obtained by composing it with the constraint automaton and
running the usual dynamic program on the product
\cite{mohri2002wfst,allauzen2007openfst}. Probabilistic suffix automata and
variable-length Markov chains already provide finite-state views of
variable-memory predictors \cite{ron1996power,buhlmann1999variable}. In that
language, the construction below is a sparse probabilistic suffix automaton, or
context trie, composed with a regular acceptor. The question in this paper is
therefore not whether weighted-automaton composition is possible in principle,
but which sparse automaton should represent a variable-order/backoff generator
in a controlled musical system, and how generation-time order policies affect
the meaning of exactness.
Related model-checking and constrained-decoding work similarly starts from a
fixed transition system or decoder and applies automata-style restrictions
\cite{baier2008principles,hokamp2017lexically,post2018fast}. Our focus is the
source-identification and sparse-compilation step for count-based
variable-order generators.

Thus the novelty here is not BP for regular or Markov constraints. It is the
variable-order extension: apply the already known BP-regular machinery to the
correct sparse state space for a variable-order/backoff generator. The question
considered here is narrower but important: how can regular control be combined
with variable-order prediction without replacing the generator by a first-order
projection, and without expanding to the dense order-\(K\) state space? A practical precedent
exists in the Continuator
\cite{pachet2003continuator}. A version with positional constraints was
developed in the European MIROR project and documented in the Routledge book on
children's creative music-making with reflexive interactive technology
\cite{rowe2017miror}. That system combined interactive variable-order
generation with positional control, but it used the support-guided mask scheme.
It could therefore enforce feasibility in real time, but did not sample with the
correct variable-order probabilities conditioned on the constraints.

The contribution is therefore a concrete variable-order/backoff instantiation
of this finite-state view, with emphasis on sparse observed contexts, correct
constrained probabilities, and the interaction with generation-time order
policies.

\paragraph{Optimal variable-order policies.}
The word ``optimal'' is overloaded in variable-order modeling. From an
information-theoretic viewpoint, the usual criterion is predictive code length
or log-loss, with a complexity penalty that prevents the maximum-order model
from simply memorizing the corpus. Rissanen's Context algorithm selects a
variable-length context tree by an MDL principle
\cite{rissanen1983universal}. Related variable-length Markov-chain methods use
penalized likelihood, BIC-like pruning, or statistical tests to keep a longer
suffix only when its continuation distribution differs sufficiently from that
of its suffixes \cite{buhlmann1999variable,ron1996power}. Context-tree
weighting avoids committing to a single pruned tree by mixing over many trees
and gives a universal coding interpretation \cite{willems1995context}. In all
these cases the object is not one globally best order \(K\), but a policy or
tree that assigns different suffix lengths to different histories.

This literature answers a model-estimation question: which context tree gives
the best prediction or compression of the source under a chosen statistical
criterion? It generally does not address finite-horizon regular control.
The question here is complementary: once a trained predictor, context tree, or
explicit generation-time order policy has been fixed, how should it be
conditioned on regular constraints? A direct reduction to an order-\(K\)
Markov chain is possible, but the dense state space has size \(|\vocab|^K\),
which is usually the wrong computational object. The corpus does not contain
all \(K\)-tuples; it contains a sparse context tree. The paper therefore asks
whether exact constrained generation can be moved from emitted symbols to
observed contexts.

The contribution is a technical decomposition of the design space:
\begin{enumerate}[leftmargin=*]
  \item a formulation of the ideal constrained distribution induced by a fixed
        variable-order/backoff generator under regular constraints;
  \item a comparison of the main tradeoffs: first-order feasibility masks,
        first-order weighted hybrids, and dense order-\(K\) lifting;
  \item an exact sparse context-state construction that performs BP on the
        product of observed variable-order contexts and a regular constraint
        automaton, with complexity linear in the sequence horizon for a fixed
        trained product graph;
  \item a separation between exact BP inference and generation-time order
        policies, including Continuator-style singleton avoidance; and
  \item a reversible-augmentation row interface, validated on 12-key
        transposition without materializing transformed corpora; and
  \item a reference implementation and evaluation based on exact
        partition-function checks on tiny regular-constrained corpora,
        scalability measurements on a larger musical example, and virtual
        augmentation checks against explicit materialization.
\end{enumerate}

\section{Background}

Let \(\vocab\) be a finite alphabet. A trained variable-order model stores
continuation counts for contexts up to a maximum length \(K\):
\[
  \ctx = (x_{t-k},\ldots,x_{t-1}), \qquad 1 \leq k \leq K,
\]
with observed counts for candidate continuations \(y \in \vocab\). At generation
time, the model searches from order \(K\) down to order \(1\) and uses the
longest context with usable continuations. This is a standard variable-memory
idea \cite{ron1996power}, used in music generation because different musical
regularities live at different local scales \cite{pearce2004improved}.
We write
\[
  N(c,y)
\]
for the number of times symbol \(y\) follows context \(c\) in the training
corpus. Thus \(N(c,\cdot)\) denotes the continuation-count vector associated
with \(c\).
We also write \(\pi_{\mathrm{ord}}\) for an order policy: a rule that decides
whether a candidate context order should be accepted or whether the sampler
should back off to a shorter suffix. The simplest policy always accepts the
first context with non-empty continuations. A Continuator-style singleton policy
may reject a high-order context when it has exactly one continuation, because
such a deterministic continuation is likely to reproduce a training fragment.
In that case the sampler tries the next shorter suffix instead.

\subsection{Vanilla variable-order backoff}

Following the constraint-programming tradition of presenting propagation
procedures explicitly, as in Mackworth's arc-consistency algorithms
\cite{mackworth1977consistency}, we write the relevant generators as algorithms
before comparing their state spaces. The usual left-to-right variable-order
sampler is the local backoff step in Algorithm~\ref{alg:vanilla-vo}.

\begin{algorithm}[t]
\caption{Vanilla variable-order backoff step}
\label{alg:vanilla-vo}
\begin{algorithmic}[1]
\Require history \(h_t=(x_0,\ldots,x_{t-1})\), counts \(N\), maximum order
  \(K\), order policy \(\pi_{\mathrm{ord}}\)
\For{\(k=K,K-1,\ldots,1\)}
  \State \(c_k \gets \suffix_k(h_t)\)
  \If{\(\sum_y N(c_k,y)=0\)}
    \State \textbf{continue}
  \EndIf
  \State \(D_k \gets \{y : N(c_k,y)>0\}\)
  \If{\(\pi_{\mathrm{ord}}\) rejects \((k,c_k,D_k)\)}
    \State \textbf{continue}
  \EndIf
  \State sample \(y\) from
    \(P(y\mid c_k)=N(c_k,y)/\sum_{y'}N(c_k,y')\)
  \State \Return \(y,k,c_k\)
\EndFor
\State \Return failure
\end{algorithmic}
\end{algorithm}

The appeal of this algorithm is precisely its backoff behavior. Long contexts
preserve local stylistic specificity; shorter contexts provide escape routes
when a long context is absent, too sparse, or too close to literal copying.

For constrained generation, a common temptation is to hack this local procedure
directly: remove candidates forbidden at the current position, force an end
symbol at a target length, reject candidates that appear to lead to dead ends,
or use a first-order BP layer as a feasibility mask. These interventions can be
useful engineering devices, but they do not by themselves define the
conditioned variable-order distribution. They modify a sequential sampler whose
decisions are local, whereas finite-horizon regular constraints require
knowledge of future probability mass.

For constrained generation it is therefore useful to expose the finite object
hidden by the vanilla procedure. The trained model induces an observed context
graph whose nodes are represented suffixes and whose labeled edges are possible
emissions. A variable-order/backoff policy then supplies the edge weights on
this graph. Raw maximum-likelihood continuation counts are one possible
weighting scheme; singleton avoidance or other backoff rules define different
stochastic processes and must be represented accordingly.

By contrast, the classical finite-length constrained Markov setting considers a
sequence
\[
  x_0,\ldots,x_{n-1}
\]
and unary constraints
\[
  x_i \in \allowed_i,
\]
where each \(\allowed_i \subseteq \vocab\) may be a singleton, a set of allowed
values, or the full vocabulary. For a first-order chain, the unconstrained
probability is
\[
  \Prfo(x_0,\ldots,x_{n-1})
  =
  \pi(x_0)\prod_{t=0}^{n-2}\Prfo(x_{t+1}\mid x_t).
\]
The constrained distribution is
\[
  \Prfo(x\mid\constraints)
  =
  \frac{1}{Z_{\fo}(\constraints)}
  \Prfo(x)
  \prod_{t=0}^{n-1}\pot_t(x_t),
  \label{eq:first-order-conditioned}
\]
where \(\pot_t(x_t)=0\) for forbidden values and \(1\) for allowed values. The
normalizing constant \(Z_{\fo}(\constraints)\), exact marginals, and exact
samples can be computed by the forward-backward algorithm.

Regular constraints replace the independent unary masks by a finite automaton
\[
  A=(Q,\vocab,\delta,q_0,F).
\]
A sequence is valid when the automaton state obtained by reading it from
\(q_0\) belongs to the accepting set \(F\). Positional constraints are recovered
by an automaton whose state is the current position. A MAXORDER or
anti-plagiarism constraint is recovered by an automaton that tracks the longest
suffix of the generated sequence that is also a prefix of a forbidden training
fragment, and rejects when a forbidden fragment is completed
\cite{papadopoulos2014avoiding}.

\section{Problem Formulation}

A variable-order model defines a conditional distribution
\[
  \Prvo(x_t \mid x_0,\ldots,x_{t-1}),
\]
by selecting a suffix of the history and sampling from the associated
continuation counts. The resulting sequence probability is
\[
  \Prvo(x_0,\ldots,x_{n-1})
  =
  \prod_{t=0}^{n-1}\Prvo(x_t \mid x_0,\ldots,x_{t-1}).
  \label{eq:vo-joint}
\]
For positional constraints, the ideal constrained target is
\[
  \Prvo(x\mid\constraints)
  =
  \frac{1}{Z_{\vo}(\constraints)}
  \Prvo(x)
  \prod_{t=0}^{n-1}\pot_t(x_t).
  \label{eq:vo-conditioned}
\]
For a regular constraint automaton \(A\), the corresponding target is
\[
  \Prvo(x\mid x\in L(A))
  =
  \frac{1}{Z_{\vo}(A)}
  \Prvo(x)\mathbf{1}[x\in L(A)].
  \label{eq:vo-regular-conditioned}
\]

The difficulty is that Equation~\eqref{eq:vo-conditioned} is not represented by
a first-order chain over emitted symbols; Equation~\eqref{eq:vo-regular-conditioned}
has the same issue. The conditional probability of \(x_t\) depends on a suffix
of the whole generated past. Thus a chain solver whose state is only \(x_t\)
can be exact for Equation~\eqref{eq:first-order-conditioned}, but not for the
variable-order targets.

\begin{observation}
Exact constrained sampling for a variable-order model requires the inference
state to contain enough information to determine both the next variable-order
conditional distribution and the current state of the constraint automaton.
\end{observation}

This observation is obvious mathematically, but useful architecturally. It says
that the issue is not the constraint language itself; it is the mismatch
between the memory of the predictor and the memory of the inference state.
First-order BP is exact for the model whose state is the last emitted symbol.
Variable-order BP must instead use the active context, or an equivalent finite
state containing the information needed by the backoff policy, and must pair it
with the finite memory of the regular constraint.

\subsection{Motivating Integer Example}
\label{sec:integer-example}

We now give a fully explicit finite example in which first-order constraint
handling is misleading precisely because the generator is variable-order. Let
the alphabet be
\[
  \vocab=\{0,1,2,3,4,5,6\}.
\]
Consider a variable-order model of maximum order \(K=2\), trained on the
following multiset \(D\) of integer sequences:
\[
\begin{array}{rcl}
10   & \times & [0,1,2,4],\\
10   & \times & [0,1,3,5],\\
1    & \times & [0,1,3,4],\\
1000 & \times & [6,2,5],\\
1000 & \times & [6,3,4].
\end{array}
\]
Assume maximum-likelihood transition estimates from observed continuation
counts and the usual longest-suffix rule. The generation problem is the
following. The fixed prefix is \((0,1)\). We generate two symbols, denoted
\(x_0,x_1\), and impose the positional constraint
\[
  x_1 = 4.
\]

\subsubsection{Feasibility is not enough}

At position \(0\), the active context is \(c=(0,1)\). The observed continuation
counts from this context come from the first three rows of \(D\): symbol \(2\)
appears \(10\) times after \((0,1)\), while symbol \(3\) appears \(10+1=11\)
times. Before applying the positional constraint, the local variable-order
predictor therefore gives
\[
  \Prvo(2\mid c)=\frac{10}{21},
  \qquad
  \Prvo(3\mid c)=\frac{11}{21}.
\]
These are not yet constrained probabilities. The positional constraint \(x_1=4\)
enters through the future mass attached to each possible first symbol. Both
candidates are feasible locally, but their constrained future masses are
determined by different order-2 contexts. After choosing \(2\), the active
context is \((1,2)\), for which the only observed continuation is \(4\). After
choosing \(3\), the active context is \((1,3)\), for which the observed
continuations are \(5\) with count \(10\) and \(4\) with count \(1\). Therefore
\[
  \Prvo(4\mid 1,2)=1,
  \qquad
  \Prvo(4\mid 1,3)=\frac{1}{11}.
\]
The unnormalized exact constrained variable-order scores for the first
generated symbol are therefore
\[
  S(x_0=2)=\frac{10}{21},
  \qquad
  S(x_0=3)=\frac{11}{21}\frac{1}{11}=\frac{1}{21}.
\]
After normalization,
\[
  \Prvo(x_0=2\mid (0,1),\;x_1=4)=\frac{10}{11},
  \qquad
  \Prvo(x_0=3\mid (0,1),\;x_1=4)=\frac{1}{11}.
\]
Both \(2\) and \(3\) are locally feasible after \((0,1)\), but they do not
leave the same amount of valid probability mass for satisfying \(x_1=4\). A
support-only method therefore cannot be distributionally correct: if it merely
keeps both candidates and samples from the local probabilities, it uses
\(10/21\) and \(11/21\) instead of the constrained probabilities \(10/11\) and
\(1/11\).

\subsubsection{Future mass must use the right state}

Now consider what a first-order constraint layer sees. Its state after choosing
\(x_0=y\) is only \(y\), so it merges all histories ending in \(y\). The
first-order projection therefore estimates
\[
  \Prfo(4\mid 2)=\frac{10}{10+1000}=\frac{1}{101},
  \qquad
  \Prfo(4\mid 3)=\frac{1+1000}{10+1+1000}=\frac{1001}{1011}.
\]
Thus the order-1 projection reverses the relevant future preference. A sampler
that weights the variable-order proposal by this first-order future mass would
use
\[
  \widetilde{S}_{\fo}(x_0=2)=\frac{10}{21}\frac{1}{101},
  \qquad
  \widetilde{S}_{\fo}(x_0=3)=\frac{11}{21}\frac{1001}{1011},
\]
which assigns approximately \(0.991\) probability to \(x_0=3\), although the
exact variable-order constrained probability of \(x_0=3\) is \(1/11\).

The example is deliberately small, but it exposes the central variable-order
problem. The first-order state has merged the contexts \((1,2)\) and \((6,2)\),
and also \((1,3)\) and \((6,3)\). Correct positional control must instead keep
the context information needed to evaluate the future mass under the same
variable-order model that generates the sequence.
Thus a weighted first-order hybrid can be worse than support alone: it uses
future mass, but future mass from the wrong state abstraction. The sparse
context graph is designed to address both failures at once. It uses future mass
rather than support alone, and it computes that mass over variable-order context
states rather than over merged first-order states.

An executable integer example using the same kind of finite-symbol setting is
available in the public Continuator repository \cite{continuatorgithub}.

\section{Tradeoffs and Reference Points}

The example above suggests several possible responses. They are useful reference
points, but they are not the main algorithmic object of this paper.

A first response is to keep BP first-order and use it only as a feasibility
mask around the vanilla variable-order sampler. This support-guided hybrid is
fast: run forward-backward on the order-1 projection, reject variable-order
candidates whose first-order future mass is zero, and sample the remaining
candidates from the selected variable-order context. Its limitation is exactly
what the example shows. It distinguishes impossible futures from possible ones,
but not fragile futures from abundant ones; moreover, feasibility is computed
after merging histories with the same last symbol.

A small variant is to weight each variable-order candidate by the first-order
future mass,
\[
  \Prvo(y\mid c_t)\,\pot_{t+1}(y)\,\beta^{(1)}_{t+1}(y),
\]
instead of using the first-order message only as a support test. This repairs
the support-only failure mode, but remains heuristic because the future mass is
still computed in the first-order projection. It is therefore not an exact
posterior for the variable-order model.

At the other end, one can make the Markov state the full order-\(K\) suffix
\[
  s_t=(x_{t-K+1},\ldots,x_t),
\]
and run forward-backward on the dense lifted chain. This is exact, and useful
as a definition or a small oracle, but its state space has size
\(|\vocab|^K\). The construction therefore loses the computational advantage of
a variable-order model, whose learned memory is sparse.

The sparse context graph introduced next occupies the useful point between
these extremes. It keeps the future-mass semantics of exact BP, but uses the
observed variable-order contexts rather than all dense \(K\)-tuples.

\section{Sparse Exact Context-State Inference}

The central construction is to lift the state space only to contexts that
actually occur in the trained model. This is the sparse object that the vanilla
variable-order sampler traverses implicitly.

\begin{definition}[Context-state set]
Let \(\contextset\) be the set of observed contexts stored by the
variable-order model, together with any start and end sentinel contexts needed
by the model. We assume that \(\contextset\) is suffix-closed for the backoff
rule: if a context is present, every suffix that the rule may fall back to is
also present, including the empty context. Equivalently, a learned context
dictionary that is not stored this way is first closed by adding the missing
suffix states. A context state \(s\in\contextset\) is a suffix of an emitted
history of length at most \(K\).
\end{definition}

The transition system is defined by emitting a symbol and canonicalizing the
new history suffix:
\[
  s \xrightarrow{\,y\,} s'
  \qquad
  \text{where}
  \qquad
  s' = \canon(s,y).
\]
Here \(\canon(s,y)\) returns the longest suffix of the concatenation \(s\cdot y\)
that belongs to \(\contextset\). Suffix-closure, together with the empty
context, makes this update deterministic and defined for every emitted symbol
that has positive probability under the policy. The graph itself specifies
support and memory: which context states exist, which emissions are possible,
and how the active context updates after each emission. A stochastic
variable-order/backoff policy then assigns edge weights. For the simple
longest-suffix model, the transition probability is
\[
  P_{\contextset}(s'\mid s,y)\Prvo(y\mid s),
\]
where \(P_{\contextset}\) is deterministic if the backoff rule maps each emitted
symbol to a unique next context. In that common case, the edge weight is simply
\[
  w(s,y,s')=\Prvo(y\mid s).
\]
Other finite-memory policies can be used in the same graph formalism, but the
edge weights must be the probabilities of the policy actually being controlled
by BP.

This includes many smoothed or mixed context distributions. If a smoothing rule
assigns a predictive probability \(P(y\mid s)\) from information contained in
the current context state and the fixed trained model, that probability can be
materialized directly as the outgoing edge weight. Kneser--Ney-style discounts,
interpolated backoff, or context-tree mixtures do not require additional
inference state merely because they are mixtures; the latent component can be
summed out before BP. Extra state is needed only when the smoothing or order
policy has generation-time memory not contained in \(s\), for example counters,
mode switches, or an adaptive model updated by the generated prefix.
For example, an interpolated suffix model can precompute
\[
  P_{\mathrm{int}}(y\mid s)
  =
  \lambda(s)\widehat P(y\mid s)
  + (1-\lambda(s))P_{\mathrm{int}}(y\mid \suffix(s)),
\]
with \(\lambda(s)\), discounts, and lower-order continuation probabilities
estimated from the fixed training corpus. The BP edge is still just
\[
  s\xrightarrow{\,y,\;P_{\mathrm{int}}(y\mid s)\,}\canon(s,y),
\]
so no mixture component has to be remembered during inference. If the escape
weight instead depends on generation-time events, such as how often a mode has
already been used, that variable must be added to the state.
Although smoothed context mixtures fit the same edge-weight formalism, our
musical use case favors explicit backoff policies: the selected order is itself
part of the generation behavior, especially when constraints such as MAXORDER
make high-order continuations infeasible or too close to quotation.

The same interface supports reversible data augmentation without copying the
training set. Let \(\aug\) be a finite set of invertible transformations of the
alphabet, extended symbolwise to contexts. Instead of inserting all transformed
sequences, the row used by the source can be computed by inverse lookup,
\[
  N_\aug(s,y)=\sum_{g\in\aug} N(g^{-1}s,g^{-1}y),
  \qquad
  P_\aug(y\mid s)=
  \frac{N_\aug(s,y)}{\sum_z N_\aug(s,z)}.
\]
For music, \(\aug\) can be the 12 pitch transpositions or another family of
reversible pitch or rhythm transforms. This gives the same source
probabilities as a materialized augmented corpus, while the original count
table remains the only stored training object; transformed rows are generated
lazily, cached, or represented by an original context plus an augmentation tag.
The regular BP recurrence is unchanged and still sums over the same augmented
paths. This does not by itself quotient the constraint product, but it exposes
where transformed product rows repeat. If the constraint language is invariant
under the same transformations, contexts can further be quotiented by orbit
representatives; if constraints mention absolute symbols, such as ending on C,
the emitted labels remain absolute while source probabilities are still
computed by orbit aggregation.

Positional constraints are attached to emitted labels, or equivalently to the
last symbol of the destination state:
\[
  \pot_{t+1}(s \xrightarrow{y} s')=\pot_{t+1}(y).
\]
Forward and backward messages are then computed over context states:
\[
  \alpha_{t+1}(s')
  =
  \sum_{s\in\contextset}
  \alpha_t(s)
  \sum_{y:s\xrightarrow{y}s'}
  \Prvo(y\mid s)\pot_{t+1}(y),
\]
and
\[
  \beta_t(s)
  =
  \sum_{y,s':s\xrightarrow{y}s'}
  \Prvo(y\mid s)\pot_{t+1}(y)\beta_{t+1}(s').
\]
The locally correct constrained sampling score is
\[
  S_{\mathrm{exact}}(y)
  =
  \Prvo(y\mid s_t)\pot_{t+1}(y)\beta_{t+1}(\canon(s_t,y)).
  \label{eq:exact-context-score}
\]
For sampling from a fixed initial context, backward messages are sufficient:
they give the remaining accepting mass after each candidate edge. Forward
messages are useful for marginals, diagnostics, and the brute-force exactness
checks, but they are not needed by the ancestral sampler itself.

\begin{proposition}
Assume that the variable-order prediction rule is represented by a deterministic
context update \(s'=\canon(s,y)\) and a transition probability
\(\Prvo(y\mid s)\). Forward-backward sampling on the sparse
context-state automaton samples exactly from
\(\Prvo(x_0,\ldots,x_{n-1}\mid\constraints)\).
\end{proposition}

\begin{proof}
The context state contains exactly the memory used by the predictor. Therefore
the probability of an emitted sequence is the product of the corresponding
context-state edge probabilities. Multiplying each edge by the unary potential
of the emitted symbol gives the unnormalized weight in
Equation~\eqref{eq:vo-conditioned}. Standard forward-backward inference on this
finite weighted automaton computes the correct conditional edge marginals and
therefore yields exact ancestral samples by backward-weighted local sampling.
\end{proof}

The size of this model is governed by the number of observed context nodes and
edges, not by \(|\vocab|^K\). If the training corpus has \(m\) tokens, the number
of inserted contexts is at most \(mK\), and in practice much smaller after
sharing. For a horizon \(n\) and a sparse context-edge set \(E_{\contextset}\),
the unconstrained forward or backward pass over the context graph costs
\[
  O(n\,|E_{\contextset}|).
\]
With regular constraints this bound is replaced by the reachable product-edge
bound in Section~\ref{sec:regular-constraints}. Thus, once the trained context
graph and acceptor are fixed, the dynamic program is linear in the requested
sequence length and in the number of reachable edges of the finite-state object
being used. The expensive part can be precomputed for a fixed source, horizon,
and constraint automaton, after which sampling is interactive.

\section{Backoff over BP Graphs and Singleton Skips}

The sparse construction requires a clean stochastic semantics. A simple
longest-suffix rule has one: for each state \(s\), choose the longest suffix
present in the context dictionary with non-empty continuation support, and use
its normalized continuation counts. The canonical state after emitting \(y\) is
the longest represented suffix of \(s\cdot y\). More generally, any deterministic
or stochastic order rule can be exact if it is part of the probabilistic model
on which BP is run.

This point is central rather than cosmetic. In a Continuator-style model,
backoff is the mechanism that mediates between high-order stylistic specificity
and lower-order freedom. A high-order context can be highly plausible precisely
because it is almost a quotation; a shorter context can be less specific but
less tied to one memorized continuation. Constrained context BP should therefore not
erase the order policy by projecting everything to order \(1\), nor should it
pretend that raw maximum-likelihood high-order counts are the only possible
notion of variable-order plausibility.

This distinction matters for a Continuator-style constrained generator. One
practical design is to prepare one BP graph for each maximum order. Let \(G_k\)
be the sparse context graph obtained by keeping contexts of length at most
\(k\), with its own backward messages \(\beta^{(k)}\) for the requested horizon
and constraints. At generation position \(t\), with current history \(h_t\), the
policy-guided backoff step is Algorithm~\ref{alg:context-bp-backoff}.

\begin{algorithm}[t]
\caption{Policy-guided context-BP backoff step}
\label{alg:context-bp-backoff}
\begin{algorithmic}[1]
\Require history \(h_t\), position \(t\), graphs \(G_1,\ldots,G_K\), backward
  messages \(\beta^{(k)}\), order policy \(\pi_{\mathrm{ord}}\)
\For{\(k=K,K-1,\ldots,1\)}
  \State \(s_k \gets \canon_k(\suffix_k(h_t))\) in \(G_k\)
  \State \(C_k \gets \emptyset\)
  \ForAll{edges \(e=(s_k \xrightarrow{y} s')\) in \(G_k\)}
    \If{\(\pot_{t+1}(y)>0\) and \(\beta^{(k)}_{t+1}(s')>0\)}
      \State \(W_k(e)\gets
        \Prvo^{(k)}(y\mid s_k)\pot_{t+1}(y)\beta^{(k)}_{t+1}(s')\)
      \State \(C_k \gets C_k \cup \{e\}\)
    \EndIf
  \EndFor
  \If{\(C_k\neq\emptyset\) and \(\pi_{\mathrm{ord}}\) accepts \(C_k\)}
    \State sample \(e=(s_k\xrightarrow{y}s')\) from \(C_k\) proportionally to
      \(W_k(e)\)
    \State \Return \(y,k,s'\)
  \EndIf
\EndFor
\State \Return failure
\end{algorithmic}
\end{algorithm}

With a longest-feasible policy, the first non-empty \(C_k\) is accepted.
Backoff occurs only when the trained model or the future constraints make the
longer context unusable.

Algorithm~\ref{alg:context-bp-backoff} therefore defines a sequential stochastic
kernel rather than the posterior of a single fixed order-\(K\) source. At step
\(t\), the policy scans the candidate orders, chooses an accepted order
\(K_t\), and then samples an outgoing edge within that order according to the
normalized weights \(W_{K_t}\). The resulting process has probability
\[
  \prod_{t=0}^{n-1}
  P_{\pi_{\mathrm{ord}}}(k_t\mid h_t,\constraints)
  \,
  P_{\mathrm{BP}}(y_t\mid h_t,k_t,\constraints),
\]
where the second factor is the BP-normalized distribution inside the selected
order. It coincides with exact conditioning of a fixed stochastic source only
when the order-selection rule has itself been included in that source's state
and transition probabilities. This is why the experiments below report a
policy success mass for the order-stack run, and a partition function only for
fixed-source runs.

More explicitly, the policy success mass is the probability, under this
sequential order-selection kernel, that the procedure reaches length \(n\)
without returning failure. For the deterministic longest-feasible policy used in
the Bach experiment, this is \(1\) exactly when every generated prefix has at
least one order with positive BP future mass. It is equal to a standard
partition function only in the special case where the same order-selection
kernel is represented as one fixed finite-state source before conditioning.

This distinction is useful rather than accidental. The context-state
construction gives an exact conditional distribution whenever the finite-memory
source is fixed in advance. The same machinery can also be used inside practical
backoff policies that deliberately change the source during generation, for
example by avoiding singleton continuations or by choosing the highest order
that leaves non-zero constrained future mass. In the latter case the claim is
not ``exact conditioning of the original maximum-order source,'' but exact BP
inside each stated policy decision. This is the semantics used by the
Continuator-style examples.

This is the constrained analogue of the vanilla backoff loop. The difference is
that each order is tested using BP information rather than only local
continuation counts. The score \(W_k\) contains the local variable-order
probability and the constrained future mass, so it avoids the support-only
trap. The backward message \(\beta^{(k)}\) is computed on context states in
\(G_k\), so it avoids the first-order trap of merging histories that should
remain distinct. Thus the mechanism keeps the variable-order machinery while
adding finite-horizon constraint satisfaction and the corresponding sampling
probabilities.

Singleton avoidance is a musically motivated variant of this policy. A
high-order context with a unique feasible continuation is precisely the case
most likely to reproduce a training fragment literally. The original
Continuator intuition is therefore to skip such contexts and try a shorter
suffix, while keeping order \(1\) as the final fallback. In a stochastic version,
the singleton can be accepted with a small order-dependent probability, for
example \(1/(k+1)\) at order \(k\), and otherwise suppressed at intermediate
orders when alternatives exist.

Such a policy is useful, but it changes the effective stochastic process. There
are two clean ways to include it in an exact statement:
\begin{enumerate}[leftmargin=*]
  \item fold it into the definition of \(\Prvo(y\mid s)\), so that the
        transition probabilities are the post-heuristic normalized weights; or
  \item represent the heuristic explicitly with additional state, if its
        behavior depends on information not contained in the context state.
\end{enumerate}
Without one of these steps, the result should be described more conservatively:
BP gives exact feasibility and conditional weights inside each candidate context
graph, and the order policy then removes or accepts candidates before
renormalization. The sampler is then a sequential constrained generation policy,
not exact conditioning of an unmodified fixed-order or maximum-likelihood
variable-order Markov source.

\section{Variable-Order Extension of Regular BP}
\label{sec:regular-constraints}

The product construction itself is standard in automaton-based regular
constraints and in BP sampling for first-order Markov chains with regular
constraints \cite{papadopoulos2015exact,pesant2004regular}. The extension here
is to replace the first-order Markov state by the sparse variable-order context
state. The positional construction above is the special case in which the
constraint state is just the current position. For a general regular
constraint, let
\[
  A=(Q,\vocab,\delta,q_0,F)
\]
be a deterministic finite automaton, where undefined transitions are rejected.
The exact inference state is the product
\[
  (s,q)\in \contextset\times Q.
\]
For every context edge \(e=(s\xrightarrow{y,p}s')\) and every automaton state
\(q\) such that \(\delta(q,y)\) is defined, the product graph contains the edge
\[
  (s,q)\xrightarrow{y,p}(s',\delta(q,y)).
\]
We use deterministic acceptors in the implementation. An \(\epsilon\)-free NFA
can be determinized before this step; running the same recursion directly on a
nondeterministic acceptor would sample accepting runs, and therefore represents
the same sequence distribution only for unambiguous automata or after the usual
determinization/summing construction.
Belief propagation is then ordinary backward dynamic programming on this
product graph. Algorithm~\ref{alg:regular-context-bp} gives the corresponding
construction and sampler.

\begin{algorithm}[t]
\caption{Regular-constrained context-BP sampler}
\label{alg:regular-context-bp}
\begin{algorithmic}[1]
\Require context graph \(G=(\contextset,E_{\contextset})\), automaton
  \(A=(Q,\vocab,\delta,q_0,F)\), horizon \(n\), initial context \(s_0\)
\State build reachable product edges
  \(E_{\otimes}=\{(s,q)\xrightarrow{y,p}(s',q') :
    (s\xrightarrow{y,p}s')\in E_{\contextset},\ q'=\delta(q,y)\}\)
\ForAll{\((s,q)\in\contextset\times Q\)}
  \State \(\beta_n(s,q)\gets \mathbf{1}[q\in F]\)
\EndFor
\For{\(t=n-1,n-2,\ldots,0\)}
  \ForAll{reachable product states \((s,q)\)}
    \State \(\beta_t(s,q)\gets 0\)
    \ForAll{product edges \((s,q)\xrightarrow{y,p}(s',q')\)}
      \State \(\beta_t(s,q)\gets
        \beta_t(s,q)+p\,\beta_{t+1}(s',q')\)
    \EndFor
  \EndFor
\EndFor
\State \(s\gets s_0,\ q\gets q_0,\ x\gets()\)
\For{\(t=0,1,\ldots,n-1\)}
  \ForAll{product edges \((s,q)\xrightarrow{y,p}(s',q')\)}
    \State \(W(y,s',q')\gets p\,\beta_{t+1}(s',q')\)
  \EndFor
  \State sample an edge proportionally to \(W\)
  \State append \(y\) to \(x\); \(s\gets s'\); \(q\gets q'\)
\EndFor
\State \Return \(x\)
\end{algorithmic}
\end{algorithm}

\begin{proposition}
Algorithm~\ref{alg:regular-context-bp} samples exactly from
\(\Prvo(x\mid x\in L(A))\), provided that the context graph edge probabilities
are the probabilities of the variable-order process being constrained.
\end{proposition}

\begin{proof}
The proof is the usual regular-BP proof applied to the context graph rather
than to a first-order Markov chain.
Each path in the product graph corresponds to exactly one emitted sequence and
one run of the automaton on that sequence. The path exists only if all emitted
symbols follow valid automaton transitions, and it ends with non-zero terminal
weight only when the automaton state is accepting. The product edge weights are
the variable-order probabilities, so the path weight is \(\Prvo(x)\). The
backward recursion computes the total remaining accepting weight from each
product state; local sampling with weights \(p\,\beta\) is therefore ancestral
sampling from the normalized accepted-path distribution.
\end{proof}

The added complexity of regular control is exactly the automaton product. If
\(E_{\otimes}\) is the set of reachable product edges and
\(V_{\otimes}\) the set of reachable product states, the backward pass and
sampling precomputation cost
\[
  O(n\,|E_{\otimes}|),
\]
with
\[
  |E_{\otimes}|\leq |Q|\,|E_{\contextset}|
\]
for a deterministic automaton. Memory is \(O(n\,|\contextset|\,|Q|)\) if all
messages are stored densely, and \(O(n\,|V_{\otimes}|)\) for a sparse reachable
table as used here. If only the partition function is required, adjacent layers
are enough. Sampling many sequences from the same source and constraint stores
the backward table once; a single streaming sample can trade memory for
recomputation. Thus positional constraints add little overhead, while MAXORDER
or anti-plagiarism constraints pay for the size of the forbidden-substring
automaton. This is the expected price of remembering both the musical context
and the constraint state.

\section{Implementation Optimizations}

The implementation separates exactness-preserving optimizations that are already
used in the reference code from more speculative accelerations. None of the
reported optimizations uses pruning, beams, bounded MDDs, or approximate
filtering: they only change the finite-state representation or the amount of
work repeated by the implementation.

The regular product is built lazily. The worst-case bound
\(|E_{\otimes}|\leq |Q|\,|E_{\contextset}|\) is useful, but many product states
are unreachable from the current prefix and horizon. The implementation creates
product states \((s,q,t)\) only when they are reached by the backward recursion
or queried during generation. The resulting cost is governed by the reachable
product,
\[
  O(n\,|E_{\otimes}^{\mathrm{reach}}|),
\]
which can be much smaller than the full Cartesian product, especially for
plagiarism automata.

Purely positional constraints are kept as time-indexed masks rather than folded
into the regular DFA. In the Bach experiment, for example, MAXORDER remains the
regular automaton, while the final-C requirement is applied at the last time
step. This avoids multiplying the MAXORDER automaton by a redundant position
automaton. Forbidden-substring and MAXORDER constraints are compiled to a dense
integer Aho--Corasick-style DFA when possible, with transitions stored by symbol.
The generic DFA path remains available for other regular constraints.

Backward values are memoized as
\[
  \beta(t,s,q)
\]
or, when an order cutoff is part of the stochastic policy,
\[
  \beta^{(k)}(t,s,q).
\]
The prepared order-stack backend caches feasible candidate sets keyed by
position, recent history, and acceptor state. Candidate sets store cumulative
weights, so drawing many samples from the same prepared backend does not rebuild
the same local distributions. The public non-trace sampling methods also avoid
allocating per-step trace records; trace-producing methods remain available for
diagnostics.

Some natural rewrites have already been tried and rejected. Replacing graph
edges by simple parallel arrays slowed the Bach MAXORDER benchmark in the tested
form. A denser key representation for the regular backward cache also regressed
BP time. These negative results suggest that future BP speedups should change
the computation model, not merely the Python object layout.

The remaining optimization roadmap is therefore split by use case. For reusable
library use, the next low-risk improvements are compiled or lazy sampling tables,
a specialized fast path for the built-in longest-feasible policy, avoiding
copies of cached candidate sets, and storing the next acceptor state in each
candidate. For BP/backward speed, a specialized MAXORDER backend or an iterative
sparse product DP may help, but both require strict equivalence tests against
the generic DFA path. A shared suffix graph would reduce duplicate context and
edge records across order cutoffs, improving memory and preparation time, but
measurements on the Bach workload show no exact overlap of time-indexed product
message states across orders; it is therefore not expected to speed the current
BP benchmark by itself. Optional NumPy or Numba backends are left as future
accelerators after the pure-Python API stabilizes.

\section{Reference Implementation}

The experiments use the \texttt{vo-regular-bp} implementation
\cite{voregbpgithub}. The public reference implementation is tagged
\texttt{v0.1.0}, ``reusable constrained order-stack backend.'' The library
builds sparse variable-order context graphs, deterministic acceptors for
positional and forbidden-substring constraints, reachable context-acceptor
products, backward messages, partition functions or policy success masses, and
exact samples. The tagged version also exposes reusable prepared backends,
constraint builders, event adapters, and a Continuator-shaped facade so that an
external system can prepare a constrained order stack once and draw many
continuations from it. The same package is used for the tiny brute-force checks
for the Bach scalability experiment, and for the virtual augmentation
validation. No approximate pruning, beam search, bounded MDD, or heuristic
filtering is used in the reported experiments.

For reproducibility, the exactness experiments were run with
\path{scripts/eval_tiny_exactness.py}, using \(20000\) samples and seed \(0\).
The CPU timing tables use the optimized implementation line that led to
\texttt{v0.1.0}; the tagged commit is \texttt{0911caf}. The Bach sweep used
\(K=1,\ldots,6\), \(n=32\), \(M=5\), \(100\) samples per \(K\), seed \(0\), one
warmup run, and five measured repeats. Timings were measured on macOS/Darwin
24.6.0 on an Apple M1 Pro with 10 cores and 32 GB RAM, using Python 3.13.11 in a
single Python process. The Bach CPU run used
\path{scripts/eval_bach_scalability.py} with
\texttt{--source-policy policy\_stack}, orders \(1,\ldots,6\), horizon \(32\),
MAXORDER \(5\), and the same sampling settings.
The virtual augmentation checks use the same Bach corpus with integer pitch
shifts \(-6,\ldots,5\), \(K=4\), horizon \(64\), fixed bar-start anchors every
eight events, and a singleton-avoiding backoff policy.
A separate positional-only implementation comparison used
\path{scripts/eval_bach_contextbp_final_compare.py} on the same machine, with
MAXORDER omitted because the public Continuator constraint problem supports
positional constraints but not forbidden-substring regular constraints.

\section{Evaluation}

We evaluate the two specific claims of the paper. First, the context-state
product defines the same constrained distribution as direct conditioning of the
variable-order source. Second, the sparse product remains tractable beyond toy
examples. The evaluation therefore has two parts: an exactness test on tiny
corpora where the partition function is computed by enumeration, and a
scalability test on a larger musical example where enumeration is infeasible.

The two parts use different but explicit source policies. The tiny exactness
experiments use the longest-suffix maximum-likelihood model from
Section~\ref{sec:integer-example}, because that source can be enumerated and
checked by hand. The main Bach experiment uses a Continuator-style constrained
order-stack policy. This policy is not exact conditioning of one fixed
stochastic source; it is an exact constrained generation policy that uses BP
future mass to choose the highest feasible order at each step. We also use a
fixed-source pure-backoff run as a diagnostic, and a smoothed suffix-mixture
source as a secondary comparison.

\subsection{Tiny exactness experiments}

The exactness experiments use tiny integer corpora, small alphabets, and short
horizons. Because enumeration is possible, we compute the constrained partition
function directly:
\[
  Z_{\mathrm{brute}}(A)=
  \sum_{x\in\vocab^n}\Prvo(x)\mathbf{1}[x\in L(A)].
\]
This gives the exact conditional target
\[
  \Prvo(x\mid x\in L(A))
  =
  \frac{\Prvo(x)\mathbf{1}[x\in L(A)]}
       {Z_{\mathrm{brute}}(A)}.
\]

The context-BP computation gives its own partition function
\[
  Z_{\mathrm{BP}}(A)=\beta_0(s_0,q_0).
\]
The first exactness check is
\[
  Z_{\mathrm{BP}}(A)=Z_{\mathrm{brute}}(A).
\]
The second check is distributional: the edge probabilities induced by
Algorithm~\ref{alg:regular-context-bp} are compared with the enumerated
conditional distribution, and empirical sample frequencies are compared by
total variation distance.

We run two variants. The first is the integer example of
Section~\ref{sec:integer-example}: prefix \((0,1)\), horizon \(2\), and the
future positional constraint that the second generated symbol is \(4\). The
accepted sequences and their probabilities are shown in
Table~\ref{tab:tiny-positional}. The method recovers the exact conditional
probabilities \(10/11\) and \(1/11\).

\begin{table}[t]
\centering
\small
\begin{tabular}{lrrrr}
\toprule
Sequence & \(\Prvo(x)\) & Exact conditional & BP conditional & Sample frequency \\
\midrule
\((2,4)\) & 0.4761904762 & 0.9090909091 & 0.9090909091 & 0.9108 \\
\((3,4)\) & 0.0476190476 & 0.0909090909 & 0.0909090909 & 0.0892 \\
\bottomrule
\end{tabular}
\caption{Tiny positional exactness test. The constraint is on the second
generated symbol, so the first choice must account for future constrained
mass.}
\label{tab:tiny-positional}
\end{table}

The second variant uses a non-positional regular constraint: generated
sequences must avoid the forbidden substring \((2,2)\). This constraint is
recognized by a three-state automaton that records whether the previous symbol
was \(2\) or whether the forbidden substring has occurred. The accepted
distribution has thirteen sequences. The largest exact masses are shown in
Table~\ref{tab:tiny-forbidden}; the full distribution is computed by the
evaluation script.

\begin{table}[t]
\centering
\small
\begin{tabular}{lrrrr}
\toprule
Sequence & \(\Prvo(x)\) & Exact conditional & BP conditional & Sample frequency \\
\midrule
\((2,3,0,1)\) & 0.1282051282 & 0.1872246696 & 0.1872246696 & 0.1912 \\
\((0,3,1,2)\) & 0.0961538462 & 0.1404185022 & 0.1404185022 & 0.14145 \\
\((3,0,1,2)\) & 0.0854700855 & 0.1248164464 & 0.1248164464 & 0.1214 \\
\((0,2,1,3)\) & 0.0769230769 & 0.1123348018 & 0.1123348018 & 0.11315 \\
\((3,2,0,1)\) & 0.0672268908 & 0.0981749528 & 0.0981749528 & 0.09775 \\
\bottomrule
\end{tabular}
\caption{Largest probabilities in the tiny non-positional regular exactness
test, where the forbidden substring \((2,2)\) is disallowed.}
\label{tab:tiny-forbidden}
\end{table}

Table~\ref{tab:tiny-summary} summarizes both experiments. In both cases the
partition function computed by context BP matches brute-force enumeration up to
floating-point precision, BP conditional probabilities match the exact
distribution, and all samples satisfy the constraint.

\begin{table}[t]
\centering
\small
\begin{tabular}{lrr}
\toprule
Metric & Positional \(x_1=4\) & Avoid \((2,2)\) \\
\midrule
\(Z_{\mathrm{brute}}\) & 0.52380952381 & 0.684766214178 \\
\(Z_{\mathrm{BP}}\) & 0.52380952381 & 0.684766214178 \\
\(|Z_{\mathrm{brute}}-Z_{\mathrm{BP}}|\) & 0 & 0 \\
\(\mathrm{TV}(\text{exact},\text{BP})\) & 0 & \(4.60\times 10^{-17}\) \\
\(\mathrm{TV}(\text{exact},\text{empirical})\) & 0.00170909 & 0.00924597 \\
Sample violations & 0 & 0 \\
Context states & 11 & 18 \\
Context edges & 23 & 40 \\
Acceptor states & 3 & 3 \\
Reachable product states & 4 & 19 \\
Time-indexed product states & 4 & 32 \\
Reachable product edges & 4 & 36 \\
\bottomrule
\end{tabular}
\caption{Summary of the tiny exactness experiments.}
\label{tab:tiny-summary}
\end{table}

\subsection{Large scalability experiment}

The second experiment uses a pitch-only encoding of Bach's Prelude in C. The
corpus has \(592\) pitch events and an alphabet of \(25\) MIDI pitches. Since
the Prelude has mostly uniform durations, the pitch projection is sufficient for
testing the state-space behavior of the algorithm.

The sweep uses \(K\in\{1,\ldots,6\}\), horizon \(n=32\), a final pitch-class
constraint requiring the sequence to end on C, and a MAXORDER constraint with
\(M=5\): no generated pitch 5-gram may occur in the training corpus. The
MAXORDER acceptor has \(657\) states in the optimized run; the final pitch-class
constraint is applied as a positional mask rather than multiplied into the
regular product. This keeps the target event unchanged while reducing the
materialized product. For each configuration, \(100\) sequences are sampled.

The main run uses constrained policy backoff. For each maximum order \(K\), the
implementation prepares context graphs for the orders \(1,\ldots,K\) and their
regular BP messages. During generation, candidate orders are considered from
high to low; orders whose regular-constrained future mass is zero are skipped;
the highest feasible order is selected; and the emitted pitch is sampled from
that order proportionally to its local continuation probability times the next
backward message. Thus the reported mass is a policy success mass, not the
partition function of one fixed pre-constraint stochastic source. In this
experiment the success mass is \(1\) for all \(K\): the policy always finds a
continuation satisfying both the final pitch-class and MAXORDER constraints.
No singleton avoidance is used in this run.

\begin{table}[t]
\centering
\scriptsize
\setlength{\tabcolsep}{3pt}
\begin{tabular}{rrrrrrrr}
\toprule
\(K\) & \(|\contextset|\) & \(|E_{\contextset}|\) & \(|Q|\) &
\(|S_{\otimes}^{\mathrm{reach}}|\) & \(|E_{\otimes}^{\mathrm{reach}}|\) &
\(|Q|\,|E_{\contextset}|\) & \(|\vocab|^K\) \\
\midrule
1 & 25 & 117 & 657 & 606 & 74115 & 76869 & 25 \\
2 & 167 & 449 & 657 & 929 & 82395 & 294993 & 625 \\
3 & 524 & 1080 & 657 & 1122 & 83162 & 709560 & 15625 \\
4 & 1180 & 2062 & 657 & 1161 & 83195 & 1354734 & 390625 \\
5 & 2187 & 3437 & 657 & 1175 & 83195 & 2258109 & 9765625 \\
6 & 3587 & 5246 & 657 & 1205 & 83207 & 3446622 & 244140625 \\
\bottomrule
\end{tabular}
\caption{Bach constrained policy-backoff experiment: sparse reachable product
compared with the full product upper bound and dense order-\(K\) state count.}
\label{tab:bach-scalability-size}
\end{table}

\begin{table}[t]
\centering
\scriptsize
\setlength{\tabcolsep}{4pt}
\begin{tabular}{rrrrrr}
\toprule
\(K\) & Success mass & BP time (s) & Sample time/seq. (ms) & Violations &
Avg. copied span \\
\midrule
1 & 1 & 0.0756 & 0.230 & 0 & 4.00 \\
2 & 1 & 0.0908 & 0.310 & 0 & 4.00 \\
3 & 1 & 0.0896 & 0.349 & 0 & 4.00 \\
4 & 1 & 0.0902 & 0.374 & 0 & 4.00 \\
5 & 1 & 0.0902 & 0.368 & 0 & 4.00 \\
6 & 1 & 0.0896 & 0.385 & 0 & 4.00 \\
\bottomrule
\end{tabular}
\caption{Bach constrained policy-backoff experiment: success mass, runtime, and
constraint satisfaction. Times are medians over five measured repeats after one
warmup. The MAXORDER constraint forbids copied 5-grams, so every sampled copied
span is at most four.}
\label{tab:bach-scalability-runtime}
\end{table}

Table~\ref{tab:bach-scalability-size} shows the central scalability result. At
\(K=6\), the reachable product has \(1205\) states and \(83207\) edges. The
full product edge upper bound is \(3446622\), and the dense lifted state space
has \(25^6=244140625\) states. Thus sparse context lifting grows with the
observed contexts and reachable constraint product, while dense lifting grows
with all possible histories. Table~\ref{tab:bach-scalability-runtime} reports
the measured backward-pass and sampling times for this implementation run. The
backward pass remains below \(0.10\) seconds in all configurations, and sampling
is below \(0.4\) ms per generated sequence. Constraint violations are zero in
every configuration.

The order traces show how the regular constraint changes the useful order. With
\(K=6\), each measured repeat samples \(100\times32=3200\) events. Between
\(3106\) and \(3119\) of those events use order \(2\), and the remaining
\(81\) to \(94\) use order \(3\). Higher orders have zero feasible future mass
at the start under MAXORDER-5 and are skipped by the constrained policy.

This is not just an implementation detail. If one instead conditions a fixed
pure longest-observed-suffix source, exact BP finds constrained mass
\(6.955\times10^{-3}\) at \(K=1\), \(3.212\times10^{-9}\) at \(K=2\), and zero
mass for \(K\geq3\) in this Bach setup. The zero-mass cases are exact: the
constraint is not relaxed. This diagnostic explains why the main musical
experiment uses constrained policy backoff rather than naive conditioning of a
fixed high-order longest-suffix source.

\paragraph{Qualitative musical illustration.}
Figure~\ref{fig:bach-anchor-example} shows a compact Bach Prelude-style sample
generated with explicit anchor constraints. The run uses horizon \(n=64\),
maximum source order \(K=4\), the longest-feasible constrained policy, and a
MAXORDER anti-copy constraint forbidding generated 6-grams that occur in the
training corpus. The anchors prescribe the pitch plan
\(C,D,E,F,G,A,B,C\) at generated positions \(0,8,\ldots,56\); in the rendered
excerpt these anchors are shown above the corresponding notes. All anchor
constraints are satisfied and the sample has zero anti-copy violations.

\begin{figure}[t]
\centering
\includegraphics[width=\textwidth]{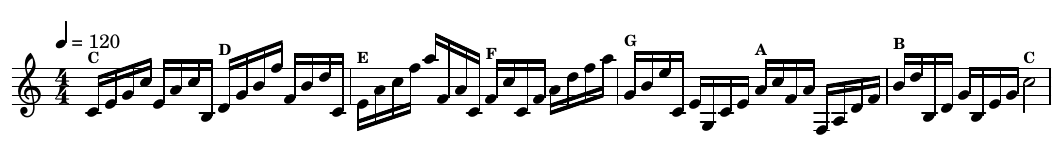}

\vspace{-0.3em}
\scriptsize
\begin{tabular}{lrrrrrrrr}
\toprule
Anchor position & 0 & 8 & 16 & 24 & 32 & 40 & 48 & 56 \\
Anchor pitch & C & D & E & F & G & A & B & C \\
Selected order & 1 & 2 & 2 & 2 & 2 & 2 & 2 & 2 \\
\bottomrule
\end{tabular}
\caption{Qualitative Bach Prelude-style sample with a prescribed cyclic anchor
plan. The selected-order trace is nearly flat: after the initial order-1 event,
all audible generated events use order \(2\). Longer unconstrained stretches, or
less frequent anchors, can give the policy more room to use higher orders.}
\label{fig:bach-anchor-example}
\end{figure}

The example is intended as a musical illustration rather than a separate
quality evaluation. It shows that the same regular-constrained machinery can
combine a planned pitch trajectory with local Prelude-like arpeggiation and an
active anti-copy language. The order trace is therefore reported as a compact
summary instead of a large plot: in this run, the musically audible part is
essentially a stable order-2 constrained continuation.

\paragraph{Virtual transposition augmentation.}
A separate validation checks that reversible augmentation can be handled at the
source-row level. In a Bach MAXORDER + final-C run, virtual 12-key
transposition stores only the original \(592\) events while matching explicit
\(7104\)-event materialized augmentation exactly: the success-mass difference
is \(0\) and the start-order masses agree. The virtual backend exposes the
same \(27371\)-state full graph semantics as materialization, but the lazy BP
graph touches only \(896\) source states in this run. Product-edge work is
unchanged in the generic exact recurrence, so this is a storage and
materialization result rather than a large BP-time speedup.

\begin{table}[t]
\centering
\scriptsize
\setlength{\tabcolsep}{3.2pt}
\begin{tabular}{lrrrrr}
\toprule
Method & Stored events & Full graph states & BP graph states &
Product edges & BP (s) \\
\midrule
Original corpus & 592 & 3587 & 3587 & 82491 & 0.020 \\
Explicit 12-key aug. & 7104 & 27371 & 27371 & 1968532 & 0.468 \\
Virtual 12-key aug. & 592 & 27371 & 896 & 1968532 & 0.497 \\
\bottomrule
\end{tabular}
\caption{Virtual augmentation matches explicit materialized augmentation
(\(0\) success-mass difference and matching start-order masses) while storing
only the original Bach sequence.}
\label{tab:virtual-summary}
\end{table}

Validation tests compare virtual augmentation against explicit materialization
for continuation rows, graph edges, positional order-stack distributions, and
regular/MAXORDER distributions. Table~\ref{tab:virtual-augmentation} reports
an anchored 12-key transposition run. All rows use \(K=4\), horizon \(64\),
integer pitch shifts \(-6,\ldots,5\), and anchors
\(60,62,64,65,67,69,71,72\) at positions \(0,8,\ldots,56\). The MAXORDER
constraint is checked against the full virtual augmented corpus. These numbers
are for the exact virtual source backend, not for a compressed product
recurrence.

\begin{table}[t]
\centering
\scriptsize
\setlength{\tabcolsep}{3.2pt}
\begin{tabular}{rrrrrrr}
\toprule
\(M\) & \(|\contextset|\) & \(|E_{\contextset}|\) &
\(|V_\otimes|\) & \(|E_\otimes|\) & Max order & Max copied span \\
\midrule
8 & 2667 & 5573 & 31790 & 2824222 & 4 & 7 \\
7 & 4226 & 7404 & 22531 & 2610565 & 3 & 6 \\
6 & 4765 & 7886 & 14394 & 2245042 & 3 & 5 \\
5 & 4803 & 7914 & 8130 & 1587777 & 2 & 4 \\
\bottomrule
\end{tabular}
\caption{Virtual 12-key transposition augmentation with anchored Bach
generation. All rows satisfy all anchors, have success mass \(1\), and have no
copied \(M\)-gram violation.}
\label{tab:virtual-augmentation}
\end{table}

Orbit diagnostics confirm that much of the enlarged product is transformed
duplicate structure, but not all of it is reusable under the finite shift set
and absolute anchors used here. Product states collapse by about \(15\times\)
modulo transposition, whereas exact row signatures collapse by about
\(1.5\times\). The implemented row cache is exact but conservative: it reuses
DFA-accepted transition rows while applying positional masks exactly, and
speeds the \(M=5\) run by \(1.37\times\), with smaller gains for \(M\geq6\).
We therefore treat deeper transformation-aware product BP as future work
requiring stronger equivariance conditions or a more symbolic recurrence.

For comparison with the existing Continuator implementation, we also ran a
positional-only benchmark with the same Bach pitch corpus, prefix, horizon, and
final-C constraint, but without MAXORDER. This is the constraint class supported
by the public Continuator constraint problem, so it is a fair implementation
comparison but not a comparison on the full regular-constraint task. The
benchmark uses Continuator's ContextBP graph compiler and backward messages
directly, with messages reused for \(100\) samples, rather than the public
sampling call that rebuilds graphs and messages for each sequence.
Table~\ref{tab:continuator-comparison} shows that \texttt{vo-regular-bp}
computes backward messages faster for all \(K\), while Continuator samples
slightly faster per sequence; end-to-end, \texttt{vo-regular-bp} is faster from
\(K=3\) upward under these timing boundaries.

\begin{table}[t]
\centering
\scriptsize
\setlength{\tabcolsep}{3pt}
\begin{tabular}{rrrrrrr}
\toprule
\(K\) & Cont. BP (ms) & New BP (ms) & Cont. sample (ms) &
New sample (ms) & Cont. total (ms) & New total (ms) \\
\midrule
1 & 0.666 & 0.250 & 0.142 & 0.163 & 15.610 & 17.130 \\
2 & 2.645 & 1.080 & 0.186 & 0.204 & 22.603 & 22.751 \\
3 & 6.528 & 2.892 & 0.233 & 0.246 & 32.109 & 30.048 \\
4 & 13.088 & 5.745 & 0.270 & 0.288 & 44.585 & 39.267 \\
5 & 22.447 & 9.852 & 0.318 & 0.333 & 60.870 & 51.024 \\
6 & 34.761 & 15.506 & 0.356 & 0.373 & 80.378 & 63.963 \\
\bottomrule
\end{tabular}
\caption{Positional-only final-C comparison with the existing Continuator
ContextBP implementation. MAXORDER is omitted here because it is not supported
by that Continuator constraint interface. Total time includes parsing, model and
graph construction, constraint construction, BP, and \(100\) samples.}
\label{tab:continuator-comparison}
\end{table}

\section{Discussion}

The experiments should be read as tests of the sparse context-state
construction, not as a benchmark against approximate first-order hybrids. The
tiny examples validate the probabilistic semantics in the only regime where the
full constrained distribution can be enumerated directly: the BP partition
function and conditional probabilities match brute force, and the samples have
accepted-sequence frequencies close to the exact distribution. These checks
support the central claim that future constrained mass must be computed in the
same variable-order state space as the generator. They do not measure the
empirical quality or speed of support-guided or first-order weighted variants.

The Bach experiment addresses the complementary computational question for a
Continuator-style policy. BP is run on sparse context--automaton products, and
the resulting future masses decide which suffix order remains usable under the
regular constraints. In this example the reachable products are orders of
magnitude smaller than both the full product upper bound and the dense lifted
state space. This does not imply that every regular constraint will lead to a
small product: the automaton for the constraint and the reachable context
structure both matter. It does show that the dense \(|\vocab|^K\) construction
is often the wrong computational baseline for a trained variable-order model or
for an order-stack constrained generation policy.

The empirical scope is intentionally narrow. The first-order support and
weighted variants are analyzed as failure modes and reference points, not
benchmarked as competing samplers, and the exact dense order-\(K\) construction
is used as a state-count baseline rather than implemented at Bach scale. The
positional-only Continuator comparison is an implementation-speed comparison on
a shared constraint class; the MAXORDER experiment is a capability and
scalability test for the new regular-constraint machinery; and the virtual
transposition experiment validates source-level augmentation against explicit
materialization, not a new approximate sampler. We report wall-clock times and
implementation-independent reachable state and edge counts, but not a separate
peak-memory profile. Broader corpora, other regular languages, determinization
stress tests, and memory-profile sweeps are natural next tests, but they are
not needed for the exactness claim.

The exactness claim is therefore deliberately conditional. For a specified
finite-memory stochastic source, it applies to the source whose transition
probabilities are the edge weights used by BP, together with the regular
constraint automaton. For an order-stack policy, it applies to the stated
sequential constrained generation semantics: BP supplies exact future mass
inside each candidate order, and the policy then chooses the highest feasible
order before normalizing over that order's feasible outgoing transitions. If
additional musical policies such as singleton avoidance are added, they must be
folded into the policy semantics or represented explicitly in the state.

The likely positive impact is improved controllability and reproducibility for
finite-state generative systems, especially in creative domains where explicit
constraints and anti-copy controls are desirable. The main risk is misuse of
controlled generation systems to imitate protected styles or produce derivative
material. The MAXORDER constraint is one technical safeguard against literal
copying, but it is not a complete policy or legal solution for style imitation.

\section{Conclusion}

Variable-order prediction and regular constraints solve complementary parts
of the controlled generation problem. Variable-order/backoff models provide
local stylistic specificity while avoiding excessive dependence on deterministic
high-order fragments; regular constraints provide finite-horizon user control,
including positional requirements and anti-plagiarism languages. Existing
first-order BP methods solve the latter problem for first-order sources, but
they do not preserve the memory used by a variable-order generator.

The central construction is to lift the inference state to the active context. Dense
order-\(K\) lifting is mathematically direct but usually too large. Sparse
context-state lifting is the variable-order analogue of the earlier BP-regular
construction: it performs forward-backward inference over the product of
observed variable-order contexts and the regular constraint automaton, while
using the same backoff semantics as the sampler.
The empirical results support the two claims made here. On tiny corpora, the
partition function computed by context BP matches brute-force enumeration, and
the induced conditional probabilities match the exact constrained distribution.
On the Bach Prelude experiment, exact BP with a final pitch-class constraint and
a MAXORDER 5-gram constraint drives a constrained order-stack policy on a
reachable product with \(1205\) states and \(83207\) edges at \(K=6\), instead
of the \(3446622\) full product edge upper bound or the \(25^6\) dense lifted
state space. This is the computational point of the construction: regular
constraint information can be computed exactly on sparse context products
without abandoning the variable-order structure that made backoff attractive in
the first place. The same row interface also handles reversible augmentation
virtually, so small corpora can expose transformed statistics without storing
transformed copies.

\bibliographystyle{plain}
\bibliography{references}

\end{document}